\documentclass{article}

\usepackage{microtype}
\usepackage{graphicx}
\usepackage{subcaption}
\usepackage{stfloats}
\usepackage{booktabs} 
\usepackage{hyperref}

\usepackage[accepted]{icml2026}

\usepackage{amsmath}
\usepackage{amssymb}
\usepackage{mathtools}
\usepackage{amsthm}

\usepackage[capitalize,noabbrev]{cleveref}

\theoremstyle{plain}
\newtheorem{theorem}{Theorem}[section]
\newtheorem{proposition}[theorem]{Proposition}
\newtheorem{lemma}[theorem]{Lemma}
\newtheorem{corollary}[theorem]{Corollary}
\theoremstyle{definition}

\theoremstyle{remark}
\newtheorem{remark}[theorem]{Remark}

\newcommand{\E}{\mathbb{E}}
\newcommand{\Prob}{\mathbb{P}}
\newcommand{\R}{\mathbb{R}}

\newcommand{\TVnorm}[1]{\left\lVert #1 \right\rVert_{\mathrm{TV}}}
\newcommand{\supnorm}[1]{\left\lVert #1 \right\rVert_\infty}
\newcommand{\norm}[1]{\left\lVert #1 \right\rVert}
\newcommand{\opnorm}[1]{\left\lVert #1 \right\rVert_{\mathrm{op}}}
\newcommand{\Cov}{\operatorname{Cov}}
\newcommand{\Var}{\operatorname{Var}}
\newcommand{\tr}{\operatorname{tr}}
\newcommand{\dd}{\,\mathrm{d}}

\usepackage[textsize=tiny]{todonotes}

\icmltitlerunning{A Diffusion Approximation for Linear TD Learning under Markovian Noise}

\begin{document}

\twocolumn[
  \icmltitle{A Diffusion Approximation for Temporal-Difference Learning \\
  with Linear Features under Markovian Noise}



  \icmlsetsymbol{equal}{*}

\begin{icmlauthorlist}
  \icmlauthor{Mattia Forzo}{tum}
  \icmlauthor{Enea Monzio Compagnoni}{equal,basel}
  \icmlauthor{Alessio Russo}{equal,bu}
  \icmlauthor{Aldo Pacchiano}{bu}
\end{icmlauthorlist}

\icmlaffiliation{tum}{Technical University of Munich (TUM), Munich, Germany}
\icmlaffiliation{basel}{University of Basel, Basel, Switzerland}
\icmlaffiliation{bu}{Boston University, Boston, USA}

\icmlcorrespondingauthor{Mattia Forzo}{mattia.forzo@tum.de}

  \icmlkeywords{temporal difference learning, stochastic differential equations, diffusion approximation}

  \vskip 0.3in
]



\printAffiliationsAndNotice{\icmlEqualContribution}

\begin{abstract}
    Temporal difference (TD) learning with linear function approximation is a core method for policy evaluation. Its classical continuous-time description is an ordinary differential equation (ODE), which captures the asymptotic mean dynamics but neglects stochastic fluctuations determining the error floor. We introduce a stochastic differential equation (SDE) approximation for linear TD(0) under Markovian noise. The resulting model distinguishes the contraction dynamics governed by the projected Bellman operator from the influence of Markovian sampling. As a consequence, the model explains the constant-stepsize error floor through the interaction between Markovian long-run covariance and the contraction geometry of the projected Bellman operator.
\end{abstract}

\section{Introduction and related work}
TD learning is a central algorithm for policy evaluation in reinforcement learning (\citet{sutton1988td,sutton2018reinforcement}), with a long line of theoretical analyses.

With linear features and a fixed policy, the asymptotic behavior of TD-learning for policy evaluation is described as a deterministic dynamical system by the linear ODE
\begin{equation}
    \dot \theta = b - A\theta,
    \label{eq:ode-main}
\end{equation}

which is at the heart of classical stochastic-approximation analysis of TD 
(\citet{tsitsiklis1997analysis, kushner2003stochastic, borkar2008stochastic, mou2024optimal, samsonov2025statistical}). The ODE approach proves where TD should go, but, by construction, it cannot capture higher-order effects.

Finite-time analyses keep the discrete recursion and give non-asymptotic error bounds under i.i.d. or Markovian data 
(\citet{korda2015td, bhandari2018finite, srikant2019finite, mitra2024simple, lee2025finite, mou2024optimal}). Such bounds typically have the form ``decaying transient plus an $O(\alpha)$ variance floor''. Recent advances also provide refined statistical and inference guarantees for linear stochastic approximation under Markovian noise (\citet{samsonov2025statistical}).

Our aim is complementary: we introduce a \emph{non-asymptotic continuous-time} model that explains the leading stochastic term producing this floor. The tool we bring to bear is the SDE approximation of stochastic algorithms, a now well-established viewpoint whose modern, non-asymptotic form rests on the \emph{weak} (distributional) approximation theory for numerical SDE schemes (\citet{mil1986weak}). Building on this, \citet{li2017stochastic,li2019stochastic} introduced \emph{stochastic modified equations}: SDEs whose drift and diffusion are fixed by matching the leading one-step moments of the algorithm, approximating the \emph{law} of the iterates---rather than their sample paths---up to an error controlled by the stepsize. Within this framework, SGD is modeled as a diffusion whose covariance is set by the gradient-noise statistics (\citet{mandt2017sgd}), and the regime in which such SDEs faithfully track the optimizer has itself been characterized (\citet{Li2021validity}). The machinery extends beyond vanilla SGD to asynchronous SGD, adaptive methods such as RMSprop and Adam, sharpness-aware minimization, minimax optimization, distributed and compressed training, differentially private training, and optimization under $(L_0,L_1)$-smoothness \citep{an2020stochastic,Malladi2022AdamSDE,compagnoni2023sde,compagnoni2024sde,compagnoni2025adaptive,compagnoni2025unbiased,compagnoni2026adaptive,compagnoni2026interaction}.

Beyond serving as approximations, these models are also analysis and design tools: they have been used with stochastic-control arguments to select learning rates and batch sizes (\citet{li2017stochastic,li2019stochastic,zhao2022batch}), to derive scaling laws tying hyperparameters to the batch size and other noise sources (\citet{jastrzkebski2017three,Malladi2022AdamSDE,compagnoni2025adaptive,compagnoni2025unbiased}), and to obtain convergence bounds, stationary approximations, and Lyapunov-style stability arguments for the loss dynamics (\citet{orvieto2019continuous,compagnoni2023sde,compagnoni2024sde,compagnoni2025adaptive,compagnoni2026interaction}). Non-Gaussian variants---L\'evy-driven (\citet{zhou2020towards}) and fractional-Brownian (\citet{lucchi2022fractional}) formulations---capture phenomena beyond Brownian diffusion at the cost of more specialized assumptions. Across all of these settings, the recurring payoff is the same: the diffusion term exposes how algorithmic noise interacts with the geometry of the problem, something the underlying ODE cannot see. We bring this viewpoint to reinforcement learning, where the driving noise is \emph{Markovian} rather than i.i.d.; identifying the correct long-run covariance is precisely what makes a weak SDE approximation of TD possible. Through this lens, we complement existing admissible-stepsize guarantees for linear TD and recover recent non-asymptotic central-limit behavior for Markovian recursions (\citet{clt_markov_td_2026}).

We analyze the constant stepsize case, from which the classical Robbins–Monro regime follows by considering vanishing diffusion. For stepsize $\alpha$, the ODE captures the $O(1)$ drift of TD, while random fluctuations around that drift are of order $\sqrt{\alpha}$. A more natural object at this scale is therefore an SDE:
\begin{equation}
    \begin{aligned}
        \dd\Theta_t&=(b-A\Theta_t)\dd t+\sqrt{\alpha}\,B(\Theta_t)\dd W_t,\\
        B(\theta)B(\theta)^\top&=\Gamma(\theta).
    \end{aligned}
    \label{eq:main-sde-intro}
\end{equation}

The change is not merely notational. Compared to the classical ODE \eqref{eq:ode-main}, a new term $B$ appears. This is the term encoding the information about the noise and whose study lets us answer questions such as: How does Markovian sampling influence TD? Which feature map or policy produces less long-run TD noise? What is the variance dynamics? Moreover, the ODE path is encoded in the drift coefficient of the SDE, so that the ODE emerges from taking the expectation of the SDE.

\paragraph{Contributions.} While prior work characterizes either asymptotic behavior via ODEs or finite-time bounds via discrete analysis, our approach provides a unified continuous-time stochastic model capturing both drift and fluctuations at finite time, adding interpretability to the known finite-time results. Technically, we use the Markov-chain Poisson equation to decompose correlated TD noise into a martingale term plus a telescoping coboundary. This identifies the diffusion covariance as the long-run covariance of the TD noise. We also construct an affine factor $B(\theta)$ of this covariance. Finally, we derive stability and local Gaussian estimates that turn the diffusion approximation into a diagnostic.



\section{Assumptions and problem setting}

\paragraph{Assumptions.}
\textbf{(A1)} The chain $(S_k)$ is finite, irreducible, and aperiodic. \textbf{(A2)} Features and rewards are bounded: $\norm{\phi(s)}\le K_\phi$ and $|r|\le K_R$. \textbf{(A3)} $A$ is positive stable, i.e. every eigenvalue of $A$ has positive real part; equivalently, the ODE matrix $-A$ is Hurwitz. In the standard discounted on-policy setting, positive stability follows from full column rank of the feature matrix in $L^2(\mu)$ \citep{tsitsiklis1997analysis}. Assumption (A3) is the same stability condition that makes the ODE converge to $\theta^\star$.

\paragraph{Markov decision process.} We consider policy evaluation for a fixed policy. The policy induces
a finite-state irreducible and aperiodic Markov chain $(S_k)_{k\ge0}$ with stationary distribution $\mu$. Rewards may be random; we assume that $R_{k+1}$ is conditionally independent of the past given $(S_k,S_{k+1})$, with law depending only on the transition $(S_k,S_{k+1})$. Thus, the observation driving the TD update $Z_k=(S_k,S_{k+1},R_{k+1})$ is itself a Markov chain. We denote by $P_Z^m$ its $m$-step transition kernel, i.e. $P_Z^m(z,A)=\Prob(Z_{k+m}\in A \mid Z_k=z)$. $P_Z$ denotes $P_Z^1$. Its mixing behavior is inherited from the state chain $(S_k)$, since the reward component does not introduce additional temporal dependence beyond $(S_k,S_{k+1})$. We use $\varrho(m) \coloneqq \sup_z\TVnorm{P_Z^m(z,\cdot)-\nu}$ to denote its  mixing profile, where $\nu$ denotes the stationary law of $(Z_k)$.

\paragraph{Temporal difference learning.} Let $\phi:\mathcal{S}\to\R^d$ be a feature map, $\gamma \in [0,1)$, and $V_\theta(s)=\phi(s)^\top\theta$. TD(0) with constant stepsize $\alpha$ is
\begin{equation}
\begin{aligned}
    \theta_{k+1}=\theta_k+\alpha\bigl(&R_{k+1}+
    \gamma\phi(S_{k+1})^\top\theta_k\\
    &-\phi(S_k)^\top\theta_k\bigr)\phi(S_k).
\end{aligned}
\label{eq:td-update}
\end{equation}
For $z=(s,s',r)$, write
\begin{equation*}
    \hat b(z)=r\,\phi(s),\qquad
    \hat A(z)=\phi(s)(\phi(s)-\gamma\phi(s'))^\top,
\end{equation*}
and define
\begin{equation*}
\begin{gathered}
    H(\theta,z)=\hat b(z)-\hat A(z)\theta,
    \qquad h(\theta)=b-A\theta,\\
    b=\E_\nu[\hat b(Z)],\qquad
    A=\E_\nu[\hat A(Z)],
\end{gathered}
\end{equation*}
where $\nu$ denotes the stationary law of $Z_k$. The TD fixed point is $\theta^\star=A^{-1}b$ when $A$ is nonsingular.

\section{From TD discrete recursion to an SDE}

The goal of this section is to construct a continuous-time stochastic model that captures both the mean dynamics and the leading-order fluctuations of the TD recursion. While classical ODE methods describe only the average behavior, they fail to account for the stochastic effects that determine the steady-state error. To address this limitation, we introduce a diffusion approximation in the form of a stochastic differential equation (SDE).

Specifically, we derive an order-$1$ weak approximation of the TD iterates in the sense of \citet{li2017stochastic}, namely:
\begin{align*}
    \max_{0 \le k \le T/\alpha}
    \left|
    \mathbb{E}\varphi(\theta_k) - \mathbb{E}\varphi(\Theta_{k\alpha})
    \right|
    &\le C \alpha, \\
    \varphi &\ \text{sufficiently regular}.
\end{align*}
Remarkably, \textit{sufficiently regular} is not very restrictive. For example, all polynomials are included, so all moments can be compared. We refer to \Cref{app:sde-derivation-clean} for details.

\paragraph{Main challenges.} Deriving such a model presents \textit{three main challenges}. First, the update noise in TD is Markovian and therefore temporally correlated, preventing a direct application of standard diffusion approximations developed for i.i.d.\ noise. Second, the diffusion term must capture the cumulative effect of these correlations. Third, the SDE must remain well-posed even when this covariance is singular, necessitating a suitable factorization of the covariance.

\paragraph{Approach.} We address these challenges as follows. Using the Poisson equation associated with the underlying Markov chain, we decompose the TD noise into a martingale component and a telescoping coboundary term, which allows us to isolate the effective diffusion covariance $\Gamma(\theta)$. We then construct an affine factorization $B(\theta)$ such that $B(\theta)B(\theta)^\top = \Gamma(\theta)$,
ensuring sufficient regularity of the SDE coefficients. This construction is crucial to guarantee well-posedness even when $\Gamma(\theta)$ is singular.

We conclude the section by stating a theorem that formalizes the weak convergence of the TD iterates to the solution of the SDE~\eqref{eq:main-sde-intro}.

\begin{table*}[t!]
\centering
\caption{Comparison of stepsize conditions in finite-time analyses of TD/linear SA under Markovian noise.}
\label{tab:stepsize-comparison}
\begin{tabular}{lccc}
\toprule
Work & Without projection & Not explicit function of $t_{\rm mix}$ & Not function of horizon $T$ \\
\midrule
\citet{korda2015td} 
& \(\checkmark\) & \(\times\) & \(\times\) \\

\citet{bhandari2018finite} 
& \(\times\) & \(\times\) & \(\checkmark\) \\

\citet{srikant2019finite} 
& \(\checkmark\) & \(\times\) & \(\times\) \\

\citet{mitra2024simple} 
& \(\checkmark\) & \(\times\) & \(\checkmark\) \\

\citet{mou2024optimal} 
& \(\checkmark\) & \(\times\) & \(\times\) \\

\citet{lee2025finite} 
& \(\checkmark\) & \(\checkmark\) & \(\times\) \\

This work
& \(\checkmark\) & \(\checkmark\) & \(\checkmark\) \\
\bottomrule
\end{tabular}
\end{table*}

\paragraph{Poisson decomposition.} Let
\begin{equation*}
    g_\theta(z)=H(\theta,z)-h(\theta),\qquad \E_\pi[g_\theta(Z)]=0.
\end{equation*}
For each fixed $\theta$, let $u_\theta$ be the centered solution of the Poisson equation
\begin{equation}
    u_\theta-P_Zu_\theta=g_\theta.
    \label{eq:poisson-main}
\end{equation}
Then the TD noise admits the decomposition
\begin{equation}
\begin{aligned}
    g_{\theta_k}(Z_k)&=u_{\theta_k}(Z_k)-u_{\theta_k}(Z_{k+1})+\xi_{k+1},\\
    \xi_{k+1}&=u_{\theta_k}(Z_{k+1})-P_Zu_{\theta_k}(Z_k),
\end{aligned}
\label{eq:poisson-split-main}
\end{equation}
where $(\xi_{k+1})$ is a martingale difference. The first two terms in \eqref{eq:poisson-split-main} form a coboundary. This is the step that converts Markovian correlation into a martingale diffusion without assuming independent samples. For details, refer to \Cref{app:sde-derivation-clean}.

\paragraph{Effective covariance.}
The effective covariance field is
\begin{equation}
\begin{aligned}
    \Gamma(\theta)=\E_\pi\!\bigl[&(u_\theta(Z_1)-P_Zu_\theta(Z_0))\\
    &\cdot (u_\theta(Z_1)-P_Zu_\theta(Z_0))^\top\bigr].
\end{aligned}
\label{eq:gamma-poisson-main}
\end{equation}
Equivalently, if $C_m(\theta)=\E_\pi[g_\theta(Z_0)g_\theta(Z_m)^\top]$, then
\begin{equation}
    \Gamma(\theta)=C_0(\theta)+\sum_{m\ge1}\bigl(C_m(\theta)+C_m(\theta)^\top\bigr).
    \label{eq:gamma-lrv-main}
\end{equation}
If the data were i.i.d., all lag terms would vanish and $\Gamma(\theta)$ would reduce to the one-step covariance. The lag terms in \eqref{eq:gamma-lrv-main} are the mathematical record of temporal dependence. They can either inflate or reduce the diffusion depending on the sign of the correlations, while slow mixing makes more terms contribute to the sum.

\begin{proposition}[Mixing controls quadratic forms of the diffusion]
\label{prop:mixing-unified}
Under \textnormal{(A1)}--\textnormal{(A2)}, there are constants $c_0,c_1>0$ depending only on the bounded TD components such that, for every $\theta\in\R^d$ and every $M\succeq 0$,
\begin{equation}
    \begin{gathered}
        \tr(M\Gamma(\theta))
        \le \tr(M)\,\tau_{\rm corr}(c_0+c_1\norm{\theta})^2, \\
        \tau_{\rm corr}=1+4\sum_{m\ge1}\varrho(m).
    \end{gathered}
    \label{eq:mixing-unified}
\end{equation}
Moreover, $\tau_{\rm corr}$ is equivalent, up to universal constants, to the usual total-variation mixing time $t_{\rm mix}(1/4)$.
\end{proposition}

A proof of Proposition \ref{prop:mixing-unified} is given in \Cref{app:mixing}. We remark that $\tau_{\rm corr}< \infty$ by \textnormal{(A1)}.


\paragraph{Worst-direction and total noise.}
Proposition~\ref{prop:mixing-unified} gives, by taking \(M=vv^\top\), $\opnorm{\Gamma(\theta)}\le\tau_{\rm corr}(c_0+c_1\norm{\theta})^2$,
while \(M=I_d\) yields $\tr(\Gamma(\theta))\le d\,\tau_{\rm corr}(c_0+c_1\norm{\theta})^2$. The first quantity is the largest directional variance of the TD noise, whereas
the second is the total injected variance. Since $\opnorm{\Gamma(\theta)}\le\tr(\Gamma(\theta))\le d\,\opnorm{\Gamma(\theta)}$, the gap between the two measures how spread out the noise is across parameter
directions. This motivates
\begin{equation}
    d_{\rm eff}(\theta)
    \coloneqq
    \frac{\tr(\Gamma(\theta))}
    {\opnorm{\Gamma(\theta)}}\in[1,d],
    \label{eq:effective-noise-dim}
\end{equation}
whenever \(\Gamma(\theta)\neq0\). Thus \(\tau_{\rm corr}\) controls the
worst-direction amplitude, while \(d_{\rm eff}(\theta)\) measures how many directions are effectively noisy.

\paragraph{A Lipschitz diffusion coefficient.}
The effective covariance $\Gamma(\theta)$ aggregates temporally correlated noise and may be singular (see \Cref{app:diffusion-factor}), so its principal square root need not be Lipschitz. For the SDE to serve as a stable and interpretable continuous-time proxy for the TD iterates, its drift and diffusion coefficients should have such regularity. To this end, we construct an affine factor $B(\theta)$.

\begin{proposition}[Affine factorization with sparsity bound]
\label{prop:affine-main}
Assume \textnormal{(A1)}--\textnormal{(A2)}. Let $n$ be the number of states and $\kappa \coloneqq \max_s \left|\{s': P^\pi(s'\mid s) >0\}\right|$. Then, there exists an integer $q\le \min\left\{d(d+1),\,2n\kappa\right\}$ and an affine map
\begin{equation}
    B:\mathbb R^d\to\mathbb R^{d\times q},
    \qquad
    B(\theta)=B_0+\sum_{i=1}^d\theta_iB_i,
    \label{eq:affine-main}
\end{equation}
such that $\Gamma(\theta)=B(\theta)B(\theta)^\top$ for all $\theta$. Hence, \(B\) is globally Lipschitz and has linear growth.
\end{proposition}

The explicit construction of such $B(\theta)$, along with its probabilistic interpretation, is given in \Cref{app:diffusion-factor}.

\paragraph{The SDE model for TD.}
\begin{theorem}[TD--SDE weak approximation]
\label{thm:diffusion-main}
Assume \textnormal{(A1)}--\textnormal{(A3)}. Fix $T,R<\infty$ and let $\theta_k^R, \Theta_{k\alpha}^R$ denote the stopped TD recursion and the stopped SDE \eqref{eq:main-sde-intro} when they leave the ball $\{\|\theta\|\le R\}$, respectively.
Then, for every sufficiently regular test function $\varphi$,
\[
    \max_{0\le k\le T/\alpha}
    \left|
        \E \varphi(\theta_k^R)
        -
        \E \varphi(\Theta_{k\alpha}^R)
    \right|
    \le C_{T,R}\,\alpha .
\]
\end{theorem}

\cref{thm:diffusion-main} makes rigorous the statement "the SDE models TD": the SDE \eqref{eq:main-sde-intro} is a weak approximation of the discrete TD recursion, i.e. the law of the SDE sampled at matching time $k\alpha$ is close to the law of the discrete iterate $\theta_k$. A proof of this fact along with a recipe to reproduce it for general discrete recursions can be found in \Cref{app:sde-derivation-clean}.

\subsection{Consequences: stability and Gaussian structure}
An immediate payoff of the SDE formulation is that stability and covariance analysis become continuous-time calculations. Let $P = P^{\top} \succ 0$ solve the Lyapunov equation $A^{\top}P+PA = I$. The conditioning of this matrix measures how strongly the drift contracts, appearing as the constant in the deterministic part of the estimate below.

\begin{theorem}[SDE stability and finite-time error]
\label{thm:stability-main}
Assume \textnormal{(A1)}--\textnormal{(A3)} and let $B$ be the affine factor in \eqref{eq:affine-main}. The SDE \eqref{eq:main-sde-intro} has a unique global strong solution for every $\alpha>0$. Moreover, for sufficiently small $\alpha$ and $E_t=\Theta_t-\theta^\star$:
\begin{equation}
    \E\norm{E_t}^2
    \le C_Ae^{-\rho_A t}\,\E\norm{E_0}^2
    +C_A\,\alpha\,\tau_{\rm corr}\, \bar d_{\rm eff},
    \label{eq:mean-square-main}
\end{equation}
where $\bar d_{\rm eff} \coloneqq \sup_{\theta:\Gamma(\theta)\neq 0} \frac{\tr(\Gamma(\theta))}{\opnorm{\Gamma(\theta)}}\le d$.

One may take $\rho_A \propto 1/\lambda_{\max}(P)$, while $C_A$ depends on the condition number of $P$, the linear-growth constant of $B$, and the norm of the bounded TD components.
\end{theorem}

This recovers the usual constant-stepsize TD picture. The added information is the interpretation of the error floor. Finite-time TD analyses prove upper bounds with comparable ingredients; the SDE explains why those ingredients appear. A proof of \cref{thm:stability-main} is presented in \Cref{app:sde-stability}.

\begin{remark}[Removing localization a posteriori]
The stopping in Theorem~\ref{thm:diffusion-main} is a technical localization device used to obtain uniform weak-error estimates on bounded sets. It is not part of the final SDE model. The localization can be removed \emph{a posteriori}; see \Cref{app:remove-localization} for a rigorous proof and a brief discussion of the impact this has on the standard analysis pipeline.
\end{remark}

\paragraph{Insights into stepsize choice.} The stepsize parameter $\alpha$ in Theorem~\ref{thm:stability-main} is the same constant stepsize as in the TD recursion. Hence, the SDE stability bound directly translates into the TD regime: for sufficiently small stepsize $\alpha$, the iterates remain stable and concentrate in an $O(\alpha)$ neighborhood of the TD fixed point. Table~\ref{tab:stepsize-comparison} places this condition in context with representative finite-time guarantees for linear TD under Markovian noise. The table should be read with some qualifications. These are discussed in \Cref{app:stepsize}.

Finally, we state a comprehensive result on the noise structure of TD. We prove finite-time estimates on how far the SDE of TD is from being an Ornstein-Uhlenbeck process and provide an explicit evolution equation of the covariance.

\begin{theorem}[Local Gaussian and covariance dynamics]
\label{thm:ou-main}
Assume \textnormal{(A1)}--\textnormal{(A3)}. Let $X_t^\alpha=(\Theta_t-\theta^\star)/\sqrt\alpha$ with $\sup_{(0,\alpha]}\E\norm{X^\alpha_0}^2 < \infty$. On every finite interval $[0,T]$, there exists a constant $C_T < \infty$, independent of $\alpha \in (0,1]$, such that
\begin{equation}
    \sup_{0\le t\le T}
    \norm{X_t^\alpha-G_t}_{L^2}
    \le C_T^{1/2}\sqrt{\alpha}.
\end{equation}
where $G$ is the Ornstein--Uhlenbeck process
\begin{equation}
    \dd G_t=-AG_t\dd t+B(\theta^\star)\dd W_t,
    \qquad
    G_0=X_0^\alpha .
    \label{eq:ou-main}
\end{equation}
Hence the leading covariance $\Sigma_t=\Cov(G_t)$ satisfies
\begin{equation}
    \dot\Sigma_t=-A\Sigma_t-\Sigma_tA^\top+\Gamma(\theta^\star).
    \label{eq:cov-main}
\end{equation}
If \textnormal{(A3)} holds, $\Sigma_t$ converges to the unique solution of $A\Sigma+\Sigma A^\top=\Gamma(\theta^\star)$.
\end{theorem}

\Cref{thm:ou-main}, whose proof is given in \Cref{app:ou-proof}, shows that near the fixed point TD errors are locally Gaussian. In particular, the covariance equation identifies both the size and the geometry of the constant-stepsize error floor. At stationarity, different directions can have different residual variances. In particular, \(\Gamma(\theta^\star)\) identifies noisy directions, while \(P\) weights
them by how slowly they are damped by the drift. Equivalently, the residual variance is governed by the alignment between Markovian noise and weakly contracting directions.

\vspace{-0.5cm}

\paragraph{Numerical verification.}
\Cref{fig:td-sde-main} compares TD with the global TD--SDE on two finite Markov reward processes.  The agreement is assessed at the ensemble level, consistent with the weak approximation in \Cref{thm:diffusion-main}; full experimental details and trajectory plots are deferred to \Cref{app:numerical-experiments}.

\begin{figure*}[ht]
\centering
\begin{subfigure}[t]{0.36\textwidth}
    \centering
    \includegraphics[width=\linewidth]{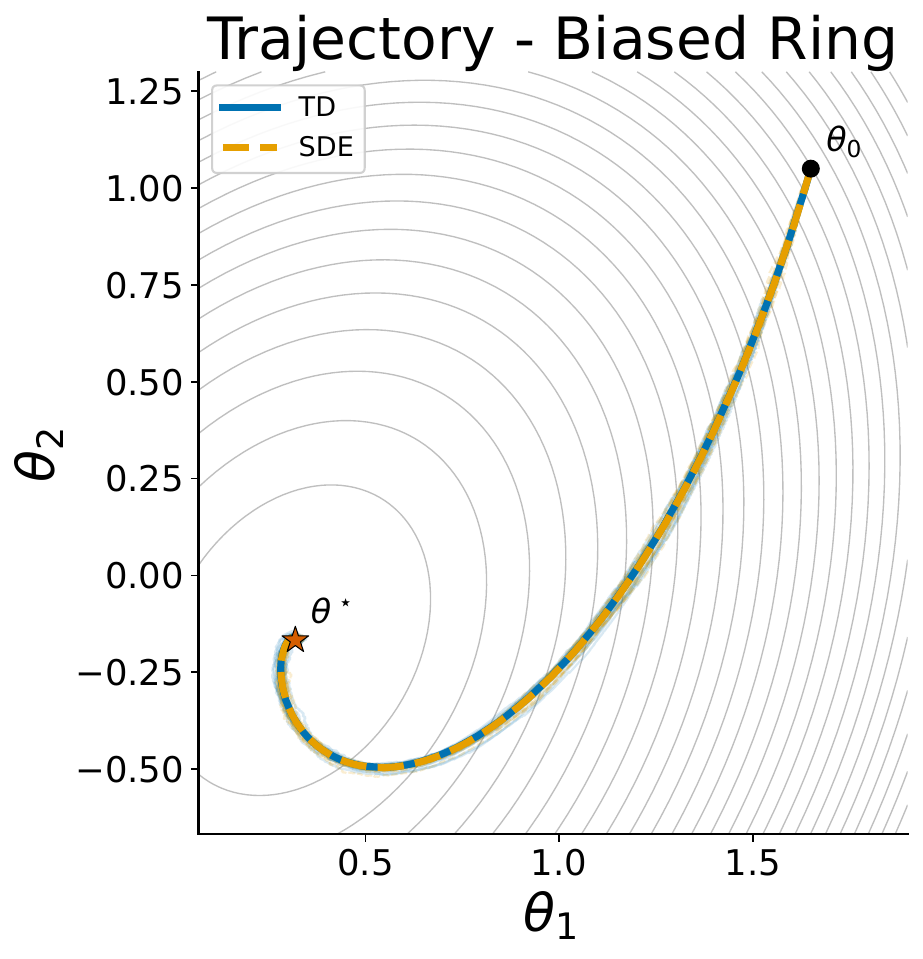}
    \caption{Empirical mean parameter trajectory.}
    \label{fig:td-sde-biased-ring-traj}
\end{subfigure}\hfill
\begin{subfigure}[t]{0.45\textwidth}
    \centering
    \includegraphics[width=\linewidth]{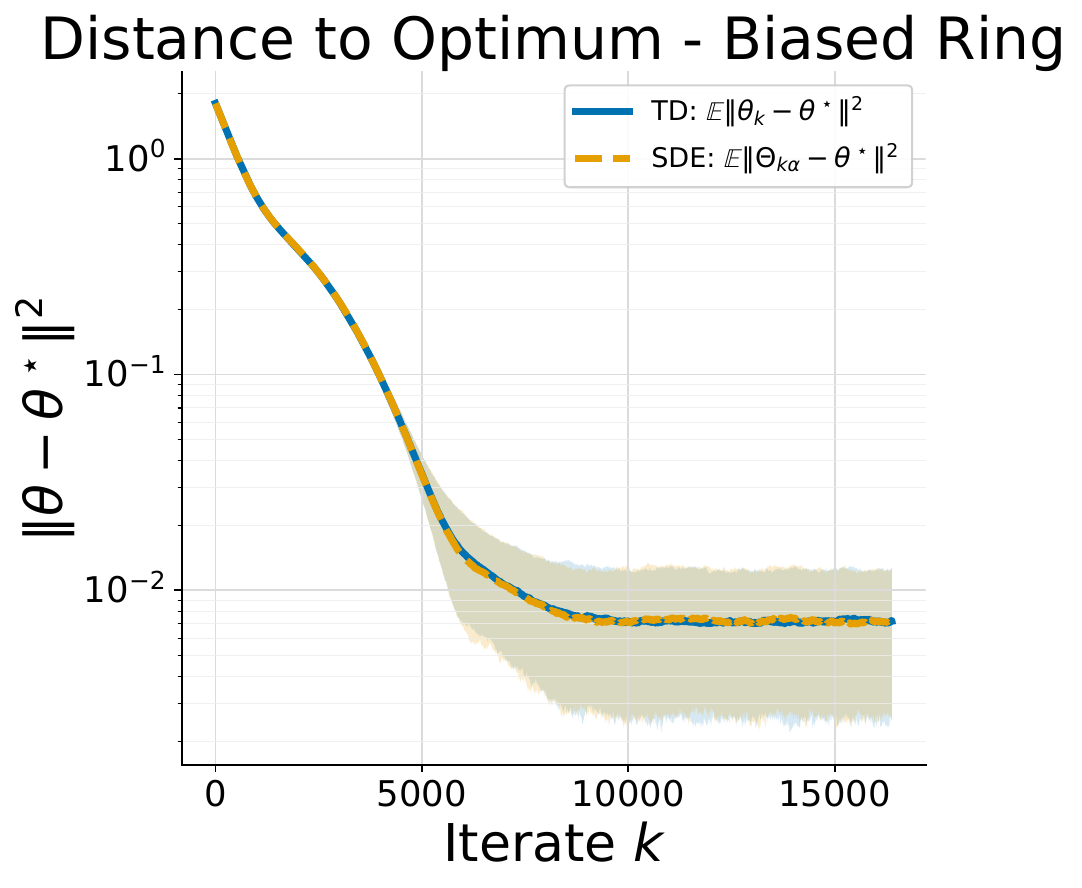}
    \caption{Empirical mean squared distance to \(\theta^\star\).}
    \label{fig:td-sde-biased-ring-dist}
\end{subfigure}

\vspace{0.35em}

\begin{subfigure}[t]{0.415\textwidth}
    \centering
    \includegraphics[width=\linewidth]{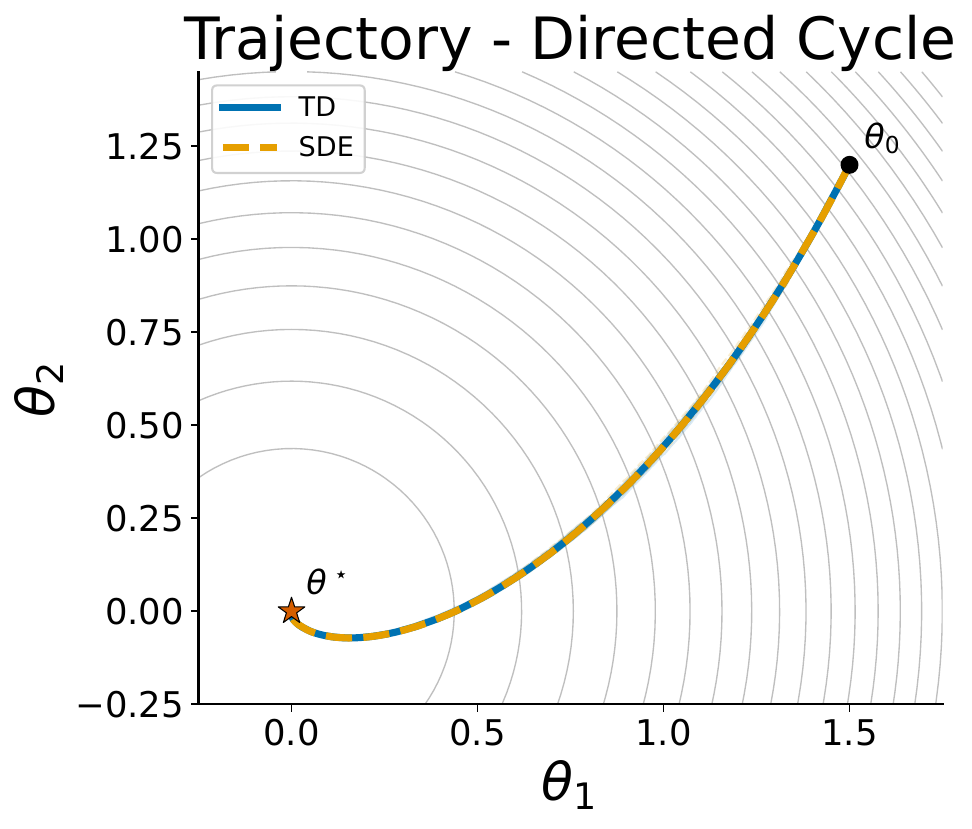}
    \caption{Empirical mean parameter trajectory.}
    \label{fig:td-sde-directed-cycle-traj}
\end{subfigure}\hfill
\begin{subfigure}[t]{0.45\textwidth}
    \centering
    \includegraphics[width=\linewidth]{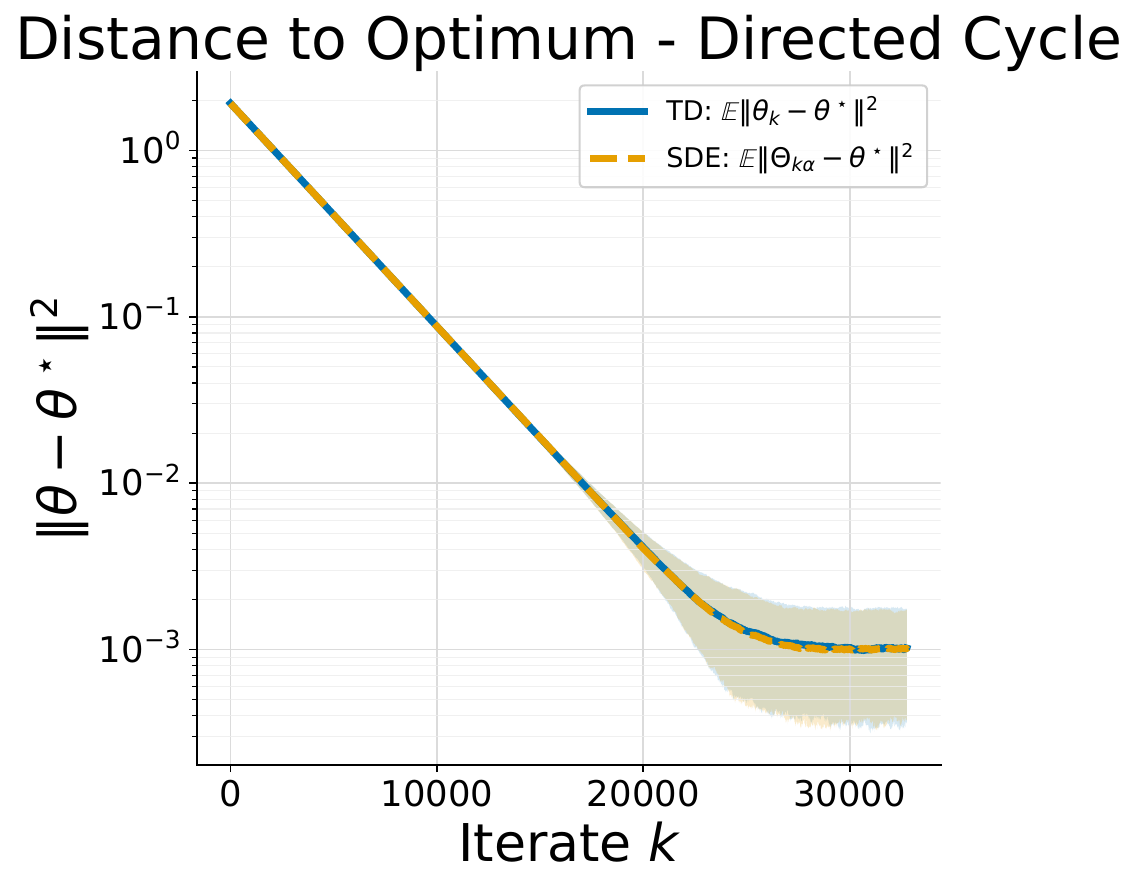}
    \caption{Empirical mean squared distance to \(\theta^\star\).}
    \label{fig:td-sde-directed-cycle-dist}
\end{subfigure}
\caption{Numerical illustration of the TD--SDE weak approximation. The left panels compare mean paths in parameter space. The right panels report empirical estimates of \(\mathbb E\|\theta_k-\theta^\star\|^2\) and \(\mathbb E\|\Theta_{k\alpha}-\theta^\star\|^2\), with shaded empirical quantile bands over independent runs.}
\label{fig:td-sde-main}
\end{figure*}


\section{Scope and conclusion}
The present note introduces an SDE modeling framework for TD(0) with linear function approximation under Markovian noise. It generalizes the classical ODE description and adds interpretability to discrete finite-time bounds by making explicit stability, covariance dynamics, and distributional properties that are otherwise inaccessible. The framework also provides a diagnostic view of admissible stepsize regimes. The technical contributions are the identification of the long-run covariance $\Gamma(\theta)$, the quantification of its dependence on mixing, and the construction of the affine factor $B(\theta)$.

\paragraph{Limitations and extensions.} 
The present version assumes bounded rewards and features. This makes error bounds more straightforward; extending the analysis to finite second moments appears natural. Further natural extensions include analysis of other reinforcement learning algorithms and general stochastic approximation methods.

\newpage

\section*{Impact Statement}

This paper presents theoretical work aimed at advancing the foundations of machine learning and stochastic approximation. It introduces tools for analyzing the stochastic dynamics of reinforcement-learning algorithms and does not involve datasets, deployed systems, or direct user-facing applications. Its likely impact is methodological: it may help researchers understand variance, stability, and covariance structure in TD and related stochastic approximation algorithms. We do not identify direct negative societal impacts beyond those associated with the downstream systems to which such algorithms may be applied.

\bibliographystyle{icml2026}
\bibliography{example_paper}

\newpage
\appendix
\onecolumn

\section{Notation, assumptions, and appendix roadmap}
\label{app:notation}

\begin{center}
\begin{minipage}{.9\textwidth}
\centering
\begin{tabular}{@{}l@{\hspace{7em}}l@{}}
    \toprule
    Symbol & Meaning \\
    \midrule
    $Z_k=(S_k,S_{k+1}, R_{k+1})$ & observation driving the TD update \\
    $\hat b(z),\hat A(z)$ & random affine components of one TD update \\
    $H(\theta,z)=\hat b(z)-\hat A(z)\theta$ & TD increment \\
    $h(\theta)=b-A\theta$ & stationary mean field \\
    $g_\theta(z)=H(\theta,z)-h(\theta)$ & centered Markovian noise \\
    $u_\theta$ & centered Poisson solution, $u_\theta-P_Zu_\theta=g_\theta$ \\
    $\xi_{k+1}$ & martingale increment from the Poisson decomposition \\
    $\Gamma(\theta)$ & long-run covariance of the TD noise \\
    $B(\theta)$ & affine factor satisfying $B(\theta)B(\theta)^\top=\Gamma(\theta)$ \\
    $\varrho(m)$ & total-variation mixing profile of $Z_k$ \\
    $\tau_{\rm corr}$ & $1+4\sum_{m\ge1}\varrho(m)$ \\
    \bottomrule
\end{tabular}
\end{minipage}
\end{center}

The appendix is organized as follows. \Cref{app:prelim} collects boundedness and Poisson-equation facts. \Cref{app:sde-derivation-clean} derives the diffusion approximation of the discrete TD recursion. \Cref{app:mixing} proves the mixing bound from Proposition \ref{prop:mixing-unified}. \Cref{app:diffusion-factor} constructs a Lipschitz covariance factor. \Cref{app:sde-stability} proves SDE stability. \Cref{app:remove-localization} shows how the localization introduced in \cref{thm:diffusion-main} is removed by means of \cref{thm:stability-main}. \Cref{app:stepsize} elaborates on the new insights about stepsize guarantees for convergence of TD. \Cref{app:ou-proof} proves finite-time Gaussian estimates and covariance dynamics of TD noise stated in \cref{thm:ou-main}. \Cref{app:numerical-experiments} gives the complete construction, numerical discretization, hyperparameters, and plotting conventions for the experiments in \Cref{fig:td-sde-main}.

\section{Preliminaries}
\label{app:prelim}

\begin{lemma}[Bounded TD components]
\label{lem:bounded-components}
Under \textnormal{(A2)},
\begin{equation*}
    \sup_z\norm{\hat b(z)}\le K_RK_\phi,
    \qquad
    \sup_z\opnorm{\hat A(z)}\le (1+\gamma)K_\phi^2 .
\end{equation*}
Consequently, there are constants $c_0,c_1,L_g<\infty$ such that
\begin{equation*}
    \norm{g(\theta,z)}\le c_0+c_1\norm{\theta},
    \qquad
    \norm{g(\theta,z)-g(\theta',z)}\le L_g\norm{\theta-\theta'} .
\end{equation*}
\end{lemma}
\begin{proof}
The first bound follows from $\norm{r(s,s')\phi(s)}\le K_RK_\phi$.  For the matrix term,
\begin{equation*}
    \opnorm{\phi(s)(\phi(s)-\gamma\phi(s'))^\top}
    \le \norm{\phi(s)}\,\norm{\phi(s)-\gamma\phi(s')}
    \le (1+\gamma)K_\phi^2 .
\end{equation*}
Since
\begin{equation*}
    g(\theta,z)=(\hat b(z)-b)-(\hat A(z)-A)\theta,
\end{equation*}
the Lipschitz and linear-growth bounds follow by taking suprema over the finite state space.
\end{proof}

Recall, that $P_Z$ is the transition kernel of the Markov chain $(Z_k)$, i.e.
\[
    P_Z(z,A)=\Prob(Z_{k+1}\in A \mid Z_k=z).
\]

For finite irreducible and aperiodic chains, the mixing profile $\varrho(m) \coloneqq \sup_z\TVnorm{P_Z^m(z,\cdot)-\pi}$ decays geometrically.  Standard Poisson-equation results for Markov chains \citep{meyn2009markov,glynn1996liapunov} therefore give the following explicit form.

\begin{lemma}[Poisson solution]
\label{lem:poisson-solution}
Assume \textnormal{(A1)}--\textnormal{(A2)}.  For each fixed $\theta$, the centered Poisson equation
\begin{equation*}
    u_\theta-P_Zu_\theta=g_\theta,\qquad \E_\pi[u_\theta]=0,
\end{equation*}
has the unique solution
\begin{equation*}
    u_\theta(z)=\sum_{m\ge0}P_Z^m g_\theta(z).
\end{equation*}
Moreover, $u_\theta$ is affine in $\theta$ and there are constants $C_u,L_u<\infty$ such that
\begin{equation*}
    \supnorm{u_\theta}\le C_u(c_0+c_1\norm{\theta}),
    \qquad
    \supnorm{u_\theta-u_{\theta'}}\le L_u\norm{\theta-\theta'} .
\end{equation*}
\end{lemma}
\begin{proof}
Since $\E_\pi[g_\theta]=0$,
\begin{equation*}
    P_Z^m g_\theta(z)=\int g_\theta(z')\bigl(P_Z^m(z,dz')-\pi(dz')\bigr),
\end{equation*}
and hence
\begin{equation*}
    \norm{P_Z^m g_\theta}_\infty
    \le 2\, \varrho(m)\supnorm{g_\theta} .
\end{equation*}
Geometric decay of $\varrho(m)$ implies absolute convergence in sup norm and gives the stated bound.  Linearity of the equation and the affine form of $g_\theta$ imply that $u_\theta$ is affine.  Applying the same bound to $g_\theta-g_{\theta'}$ gives the Lipschitz estimate.  Centering follows from stationarity of $\pi$ and termwise integration.
\end{proof}

\section{How the TD--SDE approximation is obtained}
\label{app:sde-derivation-clean}

This section has two purposes. First, it explains the general recipe for obtaining a diffusion approximation from a discrete stochastic recursion. Second, it applies this recipe to linear TD under Markovian noise. The point is not only to justify \Cref{thm:diffusion-main}, but also to make explicit the construction that turns a noisy recursion into a continuous-time SDE.

\subsection{A short recipe: from a recursion to an SDE}

We briefly recall the principle underlying stochastic modified equations and
adapt it to our setting. The key idea is that a \emph{global weak approximation}
follows from a \emph{local one-step moment matching}.

Consider a recursion of the form
\[
    \theta_{k+1}=\theta_k+\bar\Delta(\theta_k,Z_k),
\]
and a candidate SDE
\[
    \dd\Theta_t = h(\Theta_t)\dd t + \sqrt{\alpha}\,B(\Theta_t)\dd W_t.
\]
Fix a point $\theta\in\R^d$. Let
\[
    \bar\Delta(\theta)
    := \theta_1-\theta
\]
denote the one-step increment of the recursion started at $\theta$, and let
\[
    \Delta(\theta)
    := \Theta_\alpha-\theta
\]
be the increment of the SDE over a time interval of length $\alpha$.

The SDE is chosen so that these two increments match at the level of their
leading moments. Concretely, one computes:
\[
    \E[\bar\Delta(\theta)]
    \quad\text{and}\quad
    \E[\bar\Delta(\theta)\bar\Delta(\theta)^\top],
\]
and selects $h$ and $B$ so that the SDE increment satisfies
\[
    \E[\Delta(\theta)]
    =
    \alpha h(\theta)+O(\alpha^2),
\]
\[
    \E[\Delta(\theta)\Delta(\theta)^\top]
    =
    \alpha^2 h(\theta)h(\theta)^\top
    +
    \alpha^2 B(\theta)B(\theta)^\top
    +
    O(\alpha^3).
\]

If these leading terms match those of the recursion, and higher moments are
controlled, we then obtain an order-1 weak approximation by results in \citet{li2017stochastic}. In particular, for any \emph{sufficiently smooth} test function $\varphi$ of \emph{at most polynomial growth},
\[
    \max_{0\le k\le T/\alpha}
    \left|
    \E \varphi(\theta_k)-\E \varphi(\Theta_{k\alpha})
    \right|
    \le C\alpha .
\]

\paragraph{What changes under Markovian noise.}
When the noise is i.i.d., the covariance $B(\theta)B(\theta)^\top$ is simply
the one-step covariance of $\bar\Delta(\theta)$. Under Markovian sampling,
this is no longer correct: temporal correlations accumulate across time.
The correct object is the \emph{long-run covariance} of the noise process.
Identifying this covariance is the main additional step in the TD setting.
\subsection{Applying the recipe to TD}

The TD recursion can be written as
\[
    \theta_{k+1}
    =
    \theta_k+\alpha H(\theta_k,Z_k)
    =
    \theta_k+\alpha h(\theta_k)+\alpha g_{\theta_k}(Z_k),
\]
where
\[
    h(\theta)=b-A\theta,
    \qquad
    g_\theta(z)=H(\theta,z)-h(\theta),
    \qquad
    \E_\pi[g_\theta(Z)]=0 .
\]
The deterministic part immediately identifies the candidate drift: $h(\theta)=b-A\theta$.

If the variables $Z_k$ were independent, the diffusion covariance would simply be
\[
    \E_\pi[g_\theta(Z)g_\theta(Z)^\top].
\]
However, in TD the observations are generated by a Markov chain, so the noise terms are correlated. The one-step covariance misses the cumulative contribution of these correlations. The main task is therefore to identify the correct effective covariance.

\subsection{Removing Markovian dependence with the Poisson equation}

For each fixed $\theta$, let $u_\theta$ solve the centered Poisson equation
\[
    u_\theta-P_Zu_\theta=g_\theta .
\]
Then
\[
    g_{\theta_k}(Z_k)
    =
    \xi_{k+1}
    + \{u_{\theta_k}(Z_k)-u_{\theta_k}(Z_{k+1})\},
\]
where
\[
    \xi_{k+1}
    =
    u_{\theta_k}(Z_{k+1})-P_Zu_{\theta_k}(Z_k),
    \qquad
    \E[\xi_{k+1}\mid\mathcal F_k]=0 .
\]
Thus the Markovian noise is decomposed into two pieces:
\[
    \text{Markovian noise}
    =
    \text{martingale noise}
    +
    \text{coboundary}.
\]
The coboundary is a telescoping term. Indeed, for $K\le T/\alpha$,
\[
\begin{aligned}
    &\alpha\sum_{k=0}^{K-1}
    \{u_{\theta_k}(Z_k)-u_{\theta_k}(Z_{k+1})\}  \\
    &=
    \alpha\{u_{\theta_0}(Z_0)-u_{\theta_{K-1}}(Z_K)\}
    +
    \alpha\sum_{k=1}^{K-1}
    \{u_{\theta_k}(Z_k)-u_{\theta_{k-1}}(Z_k)\}.
\end{aligned}
\]
On a stopped ball $\{\|\theta\|\le R\}$, $u_\theta$ is bounded and Lipschitz in $\theta$, while $\|\theta_{k+1}-\theta_k\|=O(\alpha)$. Hence
\[
    \sup_{K\le T/\alpha}
    \left\|
    \alpha\sum_{k=0}^{K-1}
    \{u_{\theta_k}(Z_k)-u_{\theta_k}(Z_{k+1})\}
    \right\|
    =
    O(\alpha).
\]

Therefore, at the diffusion scale, TD behaves like
\[
    \theta_{k+1}
    \approx
    \theta_k+\alpha h(\theta_k)+\alpha\xi_{k+1}.
\]

\subsection{The effective covariance}

The martingale increment has conditional covariance
\[
    Q(\theta,z)
    =
    \E\!\left[
        (u_\theta(Z_1)-P_Zu_\theta(z))
        (u_\theta(Z_1)-P_Zu_\theta(z))^\top
        \mid Z_0=z
    \right].
\]
Averaging over the stationary distribution gives the effective covariance
\[
    \Gamma(\theta)
    =
    \E_\pi[Q(\theta,Z_0)].
\]
Equivalently,
\[
\begin{aligned}
    \Gamma(\theta)
    =
    \E_\pi\!\bigl[
        &(u_\theta(Z_1)-P_Zu_\theta(Z_0))  \\
        &\cdot
        (u_\theta(Z_1)-P_Zu_\theta(Z_0))^\top
    \bigr].
\end{aligned}
\]
This is also the long-run covariance of the original centered TD noise:
\[
    \Gamma(\theta)
    =
    C_0(\theta)
    +
    \sum_{m\ge1}
    \bigl(C_m(\theta)+C_m(\theta)^\top\bigr),
\]
where
\[
    C_m(\theta)=\E_\pi[g_\theta(Z_0)g_\theta(Z_m)^\top].
\]
This identity is the key structural point. The diffusion covariance is not only the instantaneous variance of TD noise; it also contains all lagged correlations created by the Markov chain.

\subsection{Moment matching for TD}

The effective one-step TD increment is
\[
    \bar\Delta(\theta)
    =
    \alpha h(\theta)+\alpha\xi_{k+1}.
\]
Its leading moments are
\[
    \E[\bar\Delta(\theta)]
    =
    \alpha h(\theta)+O(\alpha^2),
\]
and
\[
    \E[\bar\Delta(\theta)\bar\Delta(\theta)^\top]
    =
    \alpha^2 h(\theta)h(\theta)^\top
    +
    \alpha^2\Gamma(\theta)
    +
    O(\alpha^3).
\]
The candidate SDE is therefore
\[
    \dd\Theta_t
    =
    h(\Theta_t)\dd t
    +
    \sqrt{\alpha}\,B(\Theta_t)\dd W_t,
    \qquad
    B(\theta)B(\theta)^\top=\Gamma(\theta).
\]
Over a time interval of length $\alpha$, its increment satisfies the same expansions:
\[
    \E[\Delta(\theta)]
    =
    \alpha h(\theta)+O(\alpha^2),
\]
and
\[
    \E[\Delta(\theta)\Delta(\theta)^\top]
    =
    \alpha^2 h(\theta)h(\theta)^\top
    +
    \alpha^2\Gamma(\theta)
    +
    O(\alpha^3).
\]
Thus the SDE matches the leading drift and the leading covariance of the TD recursion. By the first-order moment-matching theorem for stochastic modified equations,
\[
    \max_{0\le k\le T/\alpha}
    \left|
    \E \varphi(\theta_k^{\alpha,R})
    -
    \E \varphi(\Theta_{k\alpha}^{\alpha,R})
    \right|
    \le C_{T,R}\,\alpha ,
\]
for all $\varphi$ continuous of at most polynomial growth.

\subsection{Takeaway}
The construction follows a simple conceptual recipe.

First, write the recursion as
\[
    \text{new iterate}
    =
    \text{old iterate}
    +
    \text{drift}
    +
    \text{noise}.
\]
Second, identify the correct noise covariance at the diffusion scale. For i.i.d. noise this is the one-step covariance. For Markovian noise it is the long-run covariance, obtained here through the Poisson equation. Third, choose a factor $B(\theta)$ satisfying
\[
    B(\theta)B(\theta)^\top=\Gamma(\theta).
\]
Finally, write the SDE whose one-step moments match those of the recursion:
\[
    \dd\Theta_t
    =
    h(\Theta_t)\dd t
    +
    \sqrt{\alpha}\,B(\Theta_t)\dd W_t .
\]
For linear TD under Markovian noise, this procedure yields precisely the SDE in \Cref{thm:diffusion-main}.

\section{Time correlations and diffusion size}
\label{app:mixing}

\begin{proof}[Proof of Proposition~\ref{prop:mixing-unified}]
For a unit vector $v$, define $f_{\theta,v}(z)=v^\top g_\theta(z)$.  It is centered and bounded by $\supnorm{g_\theta}$.  From the long-run covariance representation,
\begin{equation*}
    v^\top\Gamma(\theta)v
    =\E_\pi[f_{\theta,v}(Z_0)^2]
    +2\sum_{m\ge1}\E_\pi[f_{\theta,v}(Z_0)f_{\theta,v}(Z_m)].
\end{equation*}
The variance term is at most $\supnorm{g_\theta}^2$.  For $m\ge1$,
\begin{align*}
    \left|\E_\pi[f_{\theta,v}(Z_0)f_{\theta,v}(Z_m)]\right|
    &=\left|\int \pi(dz)f_{\theta,v}(z)P_Z^m f_{\theta,v}(z)\right|\\
    &\le 2\supnorm{f_{\theta,v}}^2\varrho(m)
    \le 2\supnorm{g_\theta}^2\varrho(m).
\end{align*}

Hence, for every unit vector $v$,
\[
    v^\top\Gamma(\theta)v
    \le
    \left(1+4\sum_{m\ge1}\varrho(m)\right)\supnorm{g_\theta}^2 .
\]

By Lemma~\ref{lem:bounded-components},
\[
    v^\top\Gamma(\theta)v
    \le
    \tau_{\rm corr}(c_0+c_1\norm{\theta})^2 .
\]

Now let $M\succeq0$ and write its spectral decomposition as
\[
    M=\sum_{j=1}^d \lambda_j v_jv_j^\top,
    \qquad
    \lambda_j\ge0,\quad \norm{v_j}=1 .
\]
Since $\Gamma(\theta)$ is positive semidefinite,
\[
    \tr(M\Gamma(\theta))
    =
    \sum_{j=1}^d \lambda_j v_j^\top\Gamma(\theta)v_j .
\]
Using the previous bound for each $v_j$ gives
\[
    \tr(M\Gamma(\theta))
    \le
    \sum_{j=1}^d \lambda_j\,
    \tau_{\rm corr}(c_0+c_1\norm{\theta})^2
    =
    \tr(M)\,\tau_{\rm corr}(c_0+c_1\norm{\theta})^2 .
\]

It remains to relate $\tau_{\rm corr}$ to the usual mixing time.  Let $t_{\rm mix}(\varepsilon)=\min\{m:\varrho(m)\le\varepsilon\}$ and choose $\varepsilon=1/4$.  Standard submultiplicativity of the pairwise total-variation distance gives
\begin{equation*}
    \sum_{m\ge1}\varrho(m)\le 2t_{\rm mix}(1/4),
\end{equation*}
and hence $\tau_{\rm corr}\le 1+8t_{\rm mix}(1/4)$.  Conversely, by the definition of $t_{\rm mix}(1/4)$, $\varrho(m)>1/4$ for $m<t_{\rm mix}(1/4)$, so
\begin{equation*}
    \tau_{\rm corr}=1+4\sum_{m\ge1}\varrho(m)>t_{\rm mix}(1/4).
\end{equation*}
\end{proof}

\section{A convenient choice of the diffusion factor}
\label{app:diffusion-factor}

The SDE only depends on the covariance through the identity
\[
    B(\theta)B(\theta)^\top=\Gamma(\theta).
\]
Thus $B(\theta)$ does not have to be the principal square root $\Gamma(\theta)^{1/2}$. This distinction is useful because the principal square root can be not regular enough for the SDE formulation when $\Gamma(\theta)$ is singular.

Indeed, even in dimension one, the map $a\mapsto \sqrt a$ is not Lipschitz near zero. Hence, without a uniform lower bound
\[
    \Gamma(\theta)\succeq \lambda I_d
\]
on the region of interest, Lipschitz regularity of $\Gamma(\theta)^{1/2}$ is not automatic. Such a lower bound would amount to assuming that the effective TD noise excites every parameter direction. This is a nondegeneracy condition stronger than what we need.

We instead use the square-integrability of rewards to build a different factor $B(\theta)$ that is affine in $\theta$.

\begin{proof}[Proof of Proposition~\ref{prop:affine-main}]
Write
\[
    g_\theta(z)=g^{(0)}(z)+\sum_{i=1}^d\theta_i g^{(i)}(z),
\]
where, for \(z=(s,s',r)\),
\[
    g^{(0)}(z)=r\phi(s)-b,
    \qquad
    g^{(i)}(z)=(A-\hat A(z))e_i .
\]
For \(i\ge1\),
\[
    g^{(i)}(s,s',r)
    =
    Ae_i-\phi(s)\bigl(\phi_i(s)-\gamma\phi_i(s')\bigr),
\]
so \(g^{(i)}\) depends only on the transition edge \((s,s')\), and not on the
reward sample.

By linearity of the Poisson equation,
\[
    u_\theta
    =
    u^{(0)}+\sum_{i=1}^d\theta_i u^{(i)},
    \qquad
    u^{(i)}=\sum_{m\ge0}P_Z^m g^{(i)} .
\]
For \(i\ge1\), since \(g^{(i)}\) is reward-free and the reward variable does not
affect future state transitions, every \(P_Z^m g^{(i)}\) is reward-free.
Therefore \(u^{(i)}\) is reward-free for every \(i\ge1\).

The martingale increment has the affine decomposition
\[
    \xi_\theta
    =
    \xi^{(0)}+\sum_{i=1}^d\theta_i\xi^{(i)},
    \qquad
    \xi^{(i)}
    =
    u^{(i)}(Z_1)-P_Zu^{(i)}(Z_0).
\]
For \(i\ge1\), \(\xi^{(i)}\) depends only on the state-transition structure.
Hence its scalar components are measurable with respect to the reachable edge \((S,S')\), and therefore lie in a linear space of dimension at most \(n\kappa\).

It remains to account for the \(i=0\) coefficient. Write
\[
    R=\bar r(S,S')+\varepsilon,
    \qquad
    \E[\varepsilon\mid S,S']=0 .
\]
The mean-reward part \(\bar r(S,S')\phi(S)\) is edge-measurable, so it lies in
the same edge space of dimension at most \(n\kappa\). The centered reward noise contributes functions of the form
\[
    \mathbf 1_{\{(S,S')=(s,s')\}}\varepsilon,
\]
which span at most one additional \(L^2\)-mode per reachable edge. Hence the
scalar components of all coefficients
\[
    \{\xi^{(i)}:\ 0\le i\le d\}
\]
belong to a finite-dimensional subspace \(\mathcal H\subset L^2\) with
\[
    \dim(\mathcal H)\le 2n\kappa.
\]

On the other hand, since there are \(d(d+1)\) scalar components
\[
    \{(\xi^{(i)})_r:\ 0\le i\le d,\ 1\le r\le d\},
\]
the general Hilbert-space construction gives
\[
    \dim(\mathcal H)\le d(d+1).
\]
Thus we may take
\[
    q=\dim(\mathcal H)
    \le
    \min\{d(d+1),\,2n\kappa\}.
\]

Let \(e_1,\dots,e_q\) be an orthonormal basis of \(\mathcal H\), and define
\(B(\theta)\in\mathbb R^{d\times q}\) by
\[
    B_{r\ell}(\theta)
    =
    \left\langle
        (\xi_\theta)_r,e_\ell
    \right\rangle_{L^2},
    \qquad
    1\le r\le d,\quad 1\le \ell\le q .
\]
Since \(\xi_\theta\) is affine in \(\theta\), every entry of \(B(\theta)\) is
affine in \(\theta\). Hence \(B\) is globally Lipschitz and has linear growth.

Finally, because each \((\xi_\theta)_r\in\mathcal H\), Parseval's identity gives
\[
\begin{aligned}
    (B(\theta)B(\theta)^\top)_{rs}
    &=
    \sum_{\ell=1}^q
    \left\langle
        (\xi_\theta)_r,e_\ell
    \right\rangle_{L^2}
    \left\langle
        (\xi_\theta)_s,e_\ell
    \right\rangle_{L^2}  \\
    &=
    \left\langle
        (\xi_\theta)_r,(\xi_\theta)_s
    \right\rangle_{L^2} \\
    &=
    \mathbb E\bigl[(\xi_\theta)_r(\xi_\theta)_s\bigr]
    =
    \Gamma(\theta)_{rs}.
\end{aligned}
\]
Therefore
\[
    B(\theta)B(\theta)^\top=\Gamma(\theta),
\]
yielding the desired conclusion.
\end{proof}

\begin{remark}[Edge count, sparsity, and reward-noise structure]
\label{rem:edge-vs-kappa}
Consider
\[
    E=\{(s,s'):\ P^\pi(s'|s)>0\},
    \qquad
    \kappa=\max_s |\{s':P^\pi(s'|s)>0\}|.
\]
Since each state has at most \(\kappa\) outgoing transitions and there are \(n\)
states, we always have
\[
    |E|\le n\kappa.
\]
Equality holds when every state has exactly \(\kappa\) outgoing edges; otherwise
\(|E|\) can be strictly smaller.

The refined bound
\[
    q\le 2n\kappa
\]
captures the fact that the effective noise dimension is governed by the
\emph{transition graph} rather than the ambient dimension \(d\). However, the precise contribution of the reward depends on how reward randomness is structured.

\paragraph{State-dependent vs.\ edge-dependent noise.}
There are two canonical regimes:

\begin{itemize}
    \item \textbf{State-dependent noise:}
    \[
        R=\bar r(S,S')+\varepsilon_S,
        \qquad
        \varepsilon_S\ \text{depends only on } S.
    \]
    In this case, the centered noise contributes at most one independent
    direction per state. Hence the noise dimension is controlled by \(n\),
    and one can obtain a sharper bound of order \(n\kappa\) or even \(n\)
    depending on structure.

    \item \textbf{Edge-dependent noise:}
    \[
        R=\bar r(S,S')+\varepsilon_{S,S'},
        \qquad
        \varepsilon_{S,S'}\ \text{depends on the edge}.
    \]
    Here each transition \((s,s')\) may carry its own independent noise mode.
    This yields up to one additional \(L^2\)-direction per edge, leading to the
    bound \(q\le 2|E|\).
\end{itemize}

\medskip

\paragraph{Illustration.}
Consider a state \(s\) with three possible successors \(s_1,s_2,s_3\). The two
regimes are illustrated below.

\medskip
\begin{center}
\begin{minipage}{0.45\textwidth}
\centering
\textbf{State-dependent noise.}
\medskip

\begin{tabular}{c c c}
\toprule
Edge & Mean reward $\bar r(s,s')$ & Noise \\
\midrule
$(s,s_1)$ & $5$  & $\mathcal N(0,1)$ \\
$(s,s_2)$ & $2$  & $\mathcal N(0,1)$ \\
$(s,s_3)$ & $-1$ & $\mathcal N(0,1)$ \\
\bottomrule
\end{tabular}
\end{minipage}
\hfill
\begin{minipage}{0.45\textwidth}
\centering
\textbf{Edge-dependent noise.}
\medskip

\begin{tabular}{c c c}
\toprule
Edge & Mean reward $\bar r(s,s')$ & Noise \\
\midrule
$(s,s_1)$ & $5$  & $\mathcal N(0,1)$ \\
$(s,s_2)$ & $2$  & $0$ \\
$(s,s_3)$ & $-1$ & $\text{Unif[-1,1]}$ \\
\bottomrule
\end{tabular}
\end{minipage}
\end{center}

In the state-dependent case, all transitions share the same noise source, so the number of independent noise directions does not grow
with the number of outgoing edges. In contrast, in the edge-dependent case,
each transition can introduce a distinct noise mode, and the effective
dimension scales with \(|E|\).
\end{remark}

\paragraph{Probabilistic interpretation of the Hilbertian construction.}
The Hilbert-space construction gives an intrinsic interpretation. The Brownian directions correspond to orthonormal $L^2$ modes spanning all scalar components of the martingale
increment $\xi_\theta$. Writing
\[
    W_t=(W_t^1,\dots,W_t^q),
\]
where the components are independent one-dimensional Brownian motions, the
diffusion term becomes
\[
    B(\theta)\dd W_t
    =
    \sum_{\ell=1}^q
    B_{\cdot \ell}(\theta)\dd W_t^\ell .
\]
Thus each column of $B(\theta)$ describes the loading of the TD martingale noise
onto one orthogonal noise mode $e_\ell$. The SDE noise is therefore a finite
superposition of independent Brownian perturbations associated with the
finite-dimensional $L^2$ span generated by
\[
    \bigl\{(\xi^{(i)})_r:\ 0\le i\le d,\ 1\le r\le d\bigr\}.
\]

\section{Proof of SDE stability}
\label{app:sde-stability}

\begin{proof}[Proof of Theorem~\ref{thm:stability-main}]
By Proposition~\ref{prop:affine-main}, the diffusion coefficient $B$ is affine. Hence it is globally Lipschitz and has linear growth. Since the drift
\[
    h(\theta)=b-A\theta
\]
is also globally Lipschitz and has linear growth, the SDE
\[
    \dd\Theta_t=(b-A\Theta_t)\dd t+\sqrt{\alpha}\,B(\Theta_t)\dd W_t
\]
admits a unique global strong solution for every $\alpha>0$.

It remains to prove the stability estimate. Let
\[
    \theta^\star=A^{-1}b,
    \qquad
    E_t=\Theta_t-\theta^\star .
\]
Then
\[
    \dd E_t=-AE_t\dd t+\sqrt{\alpha}\,B(\theta^\star+E_t)\dd W_t .
\]
Since $A$ is positive stable, there exists a unique symmetric positive definite matrix $P$ solving
\[
    A^\top P+PA=I .
\]
Define the Lyapunov function
\[
    V(e)=e^\top P e .
\]
Let $\lambda_{\min}(P)$ and $\lambda_{\max}(P)$ denote the extremal eigenvalues of $P$. Then
\[
    \lambda_{\min}(P)\|e\|^2
    \le V(e)
    \le
    \lambda_{\max}(P)\|e\|^2 .
\]

Applying It\^o's formula to $V(E_t)$ gives
\[
\begin{aligned}
    \mathcal L V(e)
    &=
    -2e^\top PAe
    +
    \alpha\,\tr\!\left(PB(\theta^\star+e)B(\theta^\star+e)^\top\right)  \\
    &=
    -e^\top(A^\top P+PA)e
    +
    \alpha\,\tr\!\left(P\Gamma(\theta^\star+e)\right)  \\
    &=
    -\|e\|^2
    +
    \alpha\,\tr\!\left(P\Gamma(\theta^\star+e)\right).
\end{aligned}
\]
By Proposition~\ref{prop:mixing-unified} with \(M=vv^\top\),
\[
    \opnorm{\Gamma(\theta)}
    \le
    \tau_{\rm corr}(c_0+c_1\norm{\theta})^2 .
\]
With the effective noise dimension in \eqref{eq:effective-noise-dim}, with the convention \(d_{\rm eff}(\theta)=0\) if \(\Gamma(\theta)=0\), we have
\[
    \tr(\Gamma(\theta))
    =
    d_{\rm eff}(\theta)\opnorm{\Gamma(\theta)} .
\]
Therefore,
\[
\begin{aligned}
    \tr(P\Gamma(\theta^\star+e))
    &\le
    \opnorm{P}\,\tr(\Gamma(\theta^\star+e)) \\
    &=
    \opnorm{P}\,
    d_{\rm eff}(\theta^\star+e)
    \opnorm{\Gamma(\theta^\star+e)} \\
    &\le
    \lambda_{\max}(P)\,
    \bar d_{\rm eff}\,
    \tau_{\rm corr}
    \bigl(c_0+c_1\norm{\theta^\star+e}\bigr)^2 ,
\end{aligned}
\]
where
\[
    \bar d_{\rm eff}
    \coloneqq
    \sup_{\theta:\Gamma(\theta)\neq0}d_{\rm eff}(\theta)
    \le d .
\]
Using
\[
    (c_0+c_1\norm{\theta^\star+e})^2
    \le
    2(c_0+c_1\norm{\theta^\star})^2
    +
    2c_1^2\norm{e}^2 ,
\]
we obtain
\[
    \mathcal L V(e)
    \le
    -\norm{e}^2
    +
    \alpha\,\tau_{\rm corr}\,\bar d_{\rm eff}
    \left(K_0+K_1\norm{e}^2\right),
\]
where one may take
\[
    K_0=
    2\lambda_{\max}(P)(c_0+c_1\norm{\theta^\star})^2,
    \qquad
    K_1=
    2\lambda_{\max}(P)c_1^2 .
\]
Choose \(\alpha_0>0\) small enough that
\[
    1-\alpha\tau_{\rm corr}\bar d_{\rm eff}K_1
    \ge \frac12,
    \qquad
    \forall \alpha\in(0,\alpha_0).
\]
Then
\[
    \mathcal L V(e)
    \le
    -\frac12\norm{e}^2
    +
    \alpha\,\tau_{\rm corr}\,\bar d_{\rm eff}K_0 .
\]

Using $\|e\|^2\ge V(e)/\lambda_{\max}(P)$, this implies
\[
    \mathcal L V(e)
    \le
    -\rho_A V(e)
    +
    \alpha\,\tau_{\rm corr}K_0,
    \qquad
    \rho_A:=\frac{1}{2\lambda_{\max}(P)} .
\]

Taking expectations in It\^o's formula and applying Gronwall's inequality yields
\[
    \E V(E_t)
    \le
    e^{-\rho_A t}\,\E V(E_0)
    +
    \frac{\alpha\,\tau_{\rm corr}\,\bar d_{\rm eff}K_0}{\rho_A}
    \left(1-e^{-\rho_A t}\right).
\]
Finally, converting back from $V$ to the Euclidean norm gives
\begin{align*}
    \E\|E_t\|^2
    &\le
    \frac{\lambda_{\max}(P)}{\lambda_{\min}(P)}
    e^{-\rho_A t}\,\E\|E_0\|^2
    +
    \frac{\alpha\,\tau_{\rm corr}\,\bar d_{\rm eff}K_0}{\rho_A\lambda_{\min}(P)} \\
    & = \kappa(P)\,e^{-\rho_A t}\,\E\|E_0\|^2
    + 2\kappa(P)\,\alpha\tau_{\rm corr}\bar d_{\rm eff}\,K_0,
\end{align*}
where $\kappa(P) \coloneqq \frac{\lambda_{\max}(P)}{\lambda_{\min}(P)}$ is the conditioning number of the (symmetric) matrix $P$.
Renaming constants gives
\[
    \E\norm{E_t}^2
    \le
    C_Ae^{-\rho_A t}\E\norm{E_0}^2
    +
    C_A\alpha\,\tau_{\rm corr}\,\bar d_{\rm eff}.
\]
where $C_A$ depends on the Lyapunov geometry of $A$, in particular on the condition number of $P$, and on the linear-growth constants of $B$.

The same estimate also gives
\[
    \sup_{t\ge0}\E\|\Theta_t\|^2<\infty,
\]
because $\Theta_t=E_t+\theta^\star$. Hence the unique solution to \eqref{eq:main-sde-intro} has uniformly bounded second moment.
\end{proof}

\section{Removal of localization}
\label{app:remove-localization}

The weak approximation theorem is stated for stopped processes. This section records the standard argument showing that, once global stability of the SDE is known, the localization can be removed on every finite time interval.

\begin{corollary}[Removal of localization]
\label{cor:remove-localization}
Let $(\Theta^{(R)}_t)_{t\ge0}$ denote the unique strong solution of the SDE obtained by stopping the coefficients of \eqref{eq:main-sde-intro} outside the ball $\{\|\theta\|\le R\}$, and let $(\Theta_t)_{t\ge0}$ denote the unique global strong solution of \eqref{eq:main-sde-intro}. Then for every $T>0$,
\[
    \lim_{R\to\infty}
    \Prob\Big(
    \sup_{0\le t\le T}\|\Theta^{(R)}_t-\Theta_t\|=0
    \Big)
    =1.
\]
Equivalently,
\[
    \sup_{0\le t\le T}\|\Theta^{(R)}_t-\Theta_t\|
    \xrightarrow{\Prob} 0,
    \qquad R\to\infty.
\]

Consequently, any weak limit identified for the localized processes coincides with the weak limit of the original process on every finite time interval.
\end{corollary}

\begin{proof}
Define the exit time of the full solution
\[
    \tau_R:=\inf\{t\ge 0:\|\Theta_t\|\ge R\}.
\]
By construction of the localized coefficients and pathwise uniqueness of strong solutions,
\[
    \Theta^{(R)}_t=\Theta_t
    \qquad\text{for all }t\le \tau_R .
\]
Therefore,
\[
    \sup_{0\le t\le T}\|\Theta^{(R)}_t-\Theta_t\|=0
    \qquad\text{on the event }\{\tau_R>T\}.
\]
Hence
\[
\begin{aligned}
    \Prob\Big(
    \sup_{0\le t\le T}\|\Theta^{(R)}_t-\Theta_t\|=0
    \Big)
    &\ge
    \Prob(\tau_R>T) \\
    &=
    1-\Prob\Big(\sup_{0\le t\le T}\|\Theta_t\|\ge R\Big).
\end{aligned}
\]
The final probability converges to zero as $R\to\infty$ because Theorem~\ref{thm:stability-main} gives global existence and finite moments on every finite time interval. This proves the first claim.

The convergence in probability follows immediately: for every $\varepsilon>0$,
\[
\begin{aligned}
    \Prob\Big(
    \sup_{0\le t\le T}\|\Theta^{(R)}_t-\Theta_t\|>\varepsilon
    \Big)
    &\le
    \Prob(\tau_R\le T) \\
    &=
    \Prob\Big(\sup_{0\le t\le T}\|\Theta_t\|\ge R\Big)
    \longrightarrow 0 .
\end{aligned}
\]
The final statement follows because two processes whose sup-norm distance on $[0,T]$ converges to zero in probability have the same weak limits.
\end{proof}

\paragraph{Localization argument and the reversal of the standard analysis pipeline.} The localization used above is only a proof device, standard in stochastic analysis. We are not assuming that the TD iterates are bounded as a modeling condition, nor are we modifying the algorithm by imposing boundedness. In particular, this is different from projected linear TD, where the recursion itself is changed by applying a projection step. Here the localization is introduced only temporarily to control estimates, and is then removed \emph{a posteriori} using the stability of the full SDE. Actually, this should be regarded as \textit{\textbf{a feature of the novel framework we are presenting}}. Both ODE methods and finite-time analysis need to prove some form of iterates non-explosion before starting the analysis of the algorithms recursion. The SDE approach reverses the order: first consider a process in which the iterates are artificially made bounded, then study the stability of the obtained SDE. It will be precisely this stability that, along with convergence of the algorithm, reveals that iterates were bounded in the first place and therefore the stopping device was never active.

\section{Additional discussion on stepsize comparisons}
\label{app:stepsize}

The comparison in Table~\ref{tab:stepsize-comparison} should be interpreted as a qualitative summary rather than a strict ordering of results. Existing works consider different stepsize schedules, including constant, decaying, and Robbins--Monro stepsizes. Guarantees obtained for sufficiently small constant stepsizes typically imply the corresponding decaying-stepsize behavior by allowing the stepsize to decrease over time. Thus, the table focuses on whether a result identifies a regime in which the admissible stepsize is uniform with respect to the horizon and does not explicitly require projection.

A second qualification concerns the dependence on mixing. Even when a stepsize condition is not written as an explicit function of the mixing time, it can still depend on the Markov chain through other problem-dependent constants. Different policies induce different transition kernels, and hence different temporal correlations and mixing behavior, as emphasized by \citet{nagaraj2020least}. In this sense, ``not a function of \(t_{\rm mix}\)'' means that the stated admissible regime does not require inserting an explicit upper bound on the mixing time into the stepsize choice; it does not mean that the dynamics are independent of mixing.

Our result should also be contrasted with the curvature-free slow-regime analysis of \citet{lee2025finite}. With their terminology, their stepsize condition is \textit{curvature-free in the slow regime}, whereas our SDE stability estimate is \textit{curvature-dependent} through the Lyapunov geometry of \(A\), encoded by the solution \(P\) of
\[
    A^\top P+PA=I .
\]
This dependence is natural for our purpose. The fact we recover a curvature-dependent result is due to our choice of explicitly tracking contraction and stochastic fluctuations to show how the SDE framework is able to clarify finite-time results. \emph{Nothing prevents the SDE framework to yield curvature-free results}.

\section{Proof of Gaussian estimates}
\label{app:ou-proof}

\begin{proof}[Proof of Theorem~\ref{thm:ou-main}]
We split the proof into four steps.

\smallskip

\noindent\textit{Step 1: rescaled dynamics.} Since $h(\theta^\star)=0$, dividing the SDE \eqref{eq:main-sde-intro} by $\sqrt\alpha$ gives
\begin{equation}
    \dd X_t^\alpha
    =
    -AX_t^\alpha\dd t
    +
    B(\theta^\star+\sqrt\alpha X_t^\alpha)\dd W_t .
    \label{eq:X-alpha-ou-proof}
\end{equation}

\smallskip

\noindent\textit{Step 2: comparison with the frozen Ornstein--Uhlenbeck process.}
Let $G$ solve
\[
    \dd G_t=-AG_t\dd t+B(\theta^\star)\dd W_t,
    \qquad
    G_0=X_0^\alpha .
\]
Set
\[
    D_t:=X_t^\alpha-G_t .
\]
Subtracting the dynamics of $G$ from \eqref{eq:X-alpha-ou-proof},
\begin{equation}
    \dd D_t
    =
    -AD_t\dd t
    +
    \bigl(
        B(\theta^\star+\sqrt\alpha X_t^\alpha)
        -
        B(\theta^\star)
    \bigr)\dd W_t,
    \qquad
    D_0=0 .
    \label{eq:D-ou-proof}
\end{equation}
By Proposition~\ref{prop:affine-main}, $B$ is affine, hence globally Lipschitz. Thus, for some $L_B<\infty$,
\begin{equation}
    \norm{
        B(\theta^\star+\sqrt\alpha x)-B(\theta^\star)
    }_F
    \le
    L_B\sqrt\alpha\,\norm{x}.
    \label{eq:B-lip-ou-proof}
\end{equation}

\smallskip

\noindent\textit{Step 3: finite-time $L^2$ bound.}
Applying It\^o's formula to $\norm{D_t}^2$ gives
\begin{align*}
    \dd\norm{D_t}^2
    &=
    -2\langle D_t,AD_t\rangle\dd t \\
    &\quad
    +
    \norm{
        B(\theta^\star+\sqrt\alpha X_t^\alpha)
        -
        B(\theta^\star)
    }_F^2\dd t
    +\dd M_t,
\end{align*}
where $M_t$ is a martingale. Since $A$ is fixed, there is $c_A<\infty$ such that
\[
    -2\langle x,Ax\rangle\le c_A\norm{x}^2 .
\]
Taking expectations and using \eqref{eq:B-lip-ou-proof},
\[
    \E\norm{D_t}^2
    \le
    c_A\int_0^t\E\norm{D_s}^2\dd s
    +
    L_B^2\alpha\int_0^t\E\norm{X_s^\alpha}^2\dd s .
\]

On every finite interval $[0,T]$, the SDE has finite second moments because its coefficients are globally Lipschitz with linear growth. Hence, for each $T<\infty$,
\[
    \sup_{0\le s\le T}\E\norm{X_s^\alpha}^2<\infty .
\]
Consequently, for $0\le t\le T$,
\[
    \E\norm{D_t}^2
    \le
    c_A\int_0^t\E\norm{D_s}^2\dd s
    +
    C_T\alpha .
\]
By Gr\"onwall's lemma,
\[
    \sup_{0\le t\le T}\E\norm{D_t}^2
    \le
    C_T\alpha .
\]
Equivalently,
\[
    \sup_{0\le t\le T}
    \norm{X_t^\alpha-G_t}_{L^2}
    \le
    C_T^{1/2}\sqrt\alpha .
\]

\smallskip

\noindent\textit{Step 4: covariance dynamics.}
Since $G$ is linear,
\[
    G_t
    =
    e^{-At}G_0
    +
    \int_0^t e^{-A(t-s)}B(\theta^\star)\dd W_s .
\]
Thus $G_t$ is Gaussian whenever $G_0$ is Gaussian, and its covariance $ \Sigma_t:=\Cov(G_t)$
satisfies
\[
    \dot\Sigma_t
    =
    -A\Sigma_t-\Sigma_tA^\top+\Gamma(\theta^\star),
\]
because $B(\theta^\star)B(\theta^\star)^\top=\Gamma(\theta^\star)$.

If \textnormal{(A3)} holds, then $-A$ is Hurwitz. Therefore the Lyapunov equation $A \Sigma + \Sigma A^\top = \Gamma(\theta^\star)$ has a unique solution, and the covariance flow converges to it, i.e. $\Sigma_t\to\Sigma$, concluding the proof.
\end{proof}

\paragraph{Interpretive and diagnostic value of Theorem~\ref{thm:ou-main}.} Near the fixed point, TD errors are approximated by a Gaussian law. This provides a way to compare feature maps or policies through the effective noise covariance $\Gamma(\theta^\star)$, and to estimate the residual variance after the transient phase. The covariance dynamics also shows that the error is generally anisotropic: at stationarity, \eqref{eq:cov-main} gives the directional diagnostic
\[
    \Var\left(v^{\top}(\Theta_\infty-\theta^{\star})\right) \approx \alpha v^{\top}\Sigma v,
\]
thus not only recovering the classical $O(\alpha)$ floor but also showing which directions $v$ are noisier. Analogously, we note that at stationarity $\tr(\Sigma_{\infty})=\tr(P\Gamma(\theta^\star))$, thus
\[
    \E\norm{\Theta_\infty-\theta^\star}^2
    \approx
    \alpha\,\tr(P\Gamma(\theta^\star)).
\]
The matrix $\Gamma(\theta^\star)$ identifies the noisy directions, while $P$ identifies the directions that are slow to contract under the drift $-A$. Hence $\tr(P\Gamma(\theta^\star))$ measures the amount of TD noise injected into directions that are not rapidly damped by the projected Bellman dynamics.

This motivates the contraction-weighted effective noise dimension
\begin{equation}
    d_{\rm eff}^P
    \coloneqq
    \frac{\tr(P\Gamma(\theta^\star))}
    {\opnorm{P}\,\opnorm{\Gamma(\theta^\star)}},
    \label{eq:contr-weighted-effective-noise-dim}
\end{equation}
whenever $\Gamma(\theta^\star)\neq0$. Combining this identity with Proposition~\ref{prop:mixing-unified} shows that the OU variance floor is governed not only by the mixing factor $\tau_{\rm corr}$, which controls the worst-direction noise amplitude, but also by the alignment between noisy directions and slowly contracting directions, as quantified by $d_{\rm eff}^P$.

\section{Numerical experiment details}
\label{app:numerical-experiments}

This section documents the experiments shown in \Cref{fig:td-sde-main}. Both examples are finite Markov reward processes induced by a fixed policy; there are no actions to simulate. Rewards are deterministic functions of the current state, so the observation chain used in the computation is the pair chain
\[
    Z_k=(S_k,S_{k+1}),
\]
which is the deterministic-reward specialization of the general chain \((S_k,S_{k+1},R_{k+1})\) used in the main text. All TD state chains are initialized at stationarity and all parameter paths start from the deterministic vector \(\theta_0\). TD and SDE paths are sampled independently; only average statistics are compared.

\paragraph{Mean field and TD recursion.}
For a finite state space \(\mathcal S=\{0,\ldots,n-1\}\), transition matrix \(P\), stationary distribution \(\mu\), feature map \(\phi:\mathcal S\to\mathbb R^d\), discount \(\gamma\), and deterministic reward \(r(s)\), the experiment computes
\begin{align}
    A
    &=
    \sum_{s\in\mathcal S}\mu_s\,
    \phi(s)\Bigl(\phi(s)-\gamma\sum_{s'}P_{ss'}\phi(s')\Bigr)^\top,
    \label{eq:num-A}\\
    b
    &=
    \sum_{s\in\mathcal S}\mu_s\,r(s)\phi(s),
    \label{eq:num-b}
\end{align}
and \(\theta^\star=A^{-1}b\). The TD recursion is
\begin{equation}
    \theta_{k+1}
    =
    \theta_k
    +
    \alpha
    \bigl(
        r(S_k)+\gamma\phi(S_{k+1})^\top\theta_k
        -\phi(S_k)^\top\theta_k
    \bigr)\phi(S_k).
    \label{eq:num-td}
\end{equation}

\paragraph{Pair-chain representation.}
Let
\[
    \mathcal Z=\{(s,s'):\,P_{ss'}>0\}.
\]
The stationary law and transition matrix of the pair chain are
\begin{equation}
    \pi_Z(s,s')=\mu_sP_{ss'},
    \qquad
    Q_{(s,s'),(s',s'')}=P_{s's''}.
    \label{eq:num-pair-chain}
\end{equation}
For \(z=(s,s')\), define
\[
    \widehat b(z)=r(s)\phi(s),
    \qquad
    \widehat A(z)=\phi(s)(\phi(s)-\gamma\phi(s'))^\top .
\]
The centered TD noise is affine in \(\theta\):
\begin{equation}
    g_\theta(z)
    =
    \widehat b(z)-b+(A-\widehat A(z))\theta
    \eqqcolon c_z+L_z\theta .
    \label{eq:num-centered-noise}
\end{equation}

\paragraph{SDE discretization.}
The global TD--SDE is simulated on the matched grid \(t_k=k\alpha\) by Euler--Maruyama:
\begin{equation}
    \Theta_{k+1}
    =
    \Theta_k
    +
    \alpha(b-A\Theta_k)
    +
    \alpha B(\Theta_k)\zeta_k,
    \qquad
    \zeta_k\sim\mathcal N(0,I_q).
    \label{eq:num-em}
\end{equation}
The factor \(\alpha\) in the stochastic term follows from the SDE coefficient \(\sqrt\alpha B(\Theta_t)\) and the Brownian increment \(\Delta W_k\sim\mathcal N(0,\alpha I_q)\).

\paragraph{Experiment instances.}
The two examples in \Cref{fig:td-sde-main} are generated as follows.
\begin{table}[H]
\centering
\label{tab:num-experiment-parameters}
\begin{tabular}{lll}
\toprule
Quantity & Biased  ring &  directed cycle \\
\midrule
States/features & \(n=5,\ d=2\) & \(n=8,\ d=2\) \\
Discount & \(\gamma=0.95\) & \(\gamma=0.50\) \\
Stepsize & \(\alpha=10^{-3}\) & \(\alpha=10^{-3}\) \\
Initialization & \((1.65,1.05)\) & \((1.55,1.20)\) \\
Runs & \(M=1024\) & \(M=1024\) \\
Iterations & \(K=1.6\times 10^4\) & \(K=3.3\times 10^4\) \\
Seed & 178 & 42 \\
Transitions & \(P_{s,s+1}=0.88,\ P_{s,s}=0.08,\ P_{s,s-1}=0.04\) & \(P_{s,s+1}=0.80,\ P_{s,s}=0.20\) \\
Features & \(\Phi_{s\ell}\sim\mathcal N(0,1)\) once & \(\phi(s)=(\cos\omega_s,\sin\omega_s)\) \\
Reward & \(r(s)\sim\mathcal N(0,1)\) once & \(r(s)=0.08\cos(2\omega_s)\) \\
Angles & -- & \(\omega_s=2\pi s/n\) \\
Pair states & \(|\mathcal Z|=15\) & \(|\mathcal Z|=16\) \\
\bottomrule
\end{tabular}
\caption{Finite Markov reward processes and simulation parameters for \Cref{fig:td-sde-main}. The biased ring uses random features and rewards drawn once with the stated seed and then kept fixed.}
\end{table}

\paragraph{Plotting convention.}
The trajectory panels plot empirical mean paths in parameter space for TD and for the global TD--SDE. The gray contours are level sets of \((\theta-\theta^\star)^\top P_L(\theta-\theta^\star)\), where \(P_L\) solves the Lyapunov equation \(A^\top P_L+P_LA=I\). The distance panels plot empirical estimates of \(\mathbb E\|\theta_k-\theta^\star\|^2\) and \(\mathbb E\|\Theta_{k\alpha}-\theta^\star\|^2\) on a logarithmic scale; shaded regions are empirical 10--90\% quantile bands over the independent runs. These plots should not be read as a pathwise coupling between TD and the SDE. They compare ensemble statistics, matching the weak approximation statement in \Cref{thm:diffusion-main}.

\end{document}